\documentclass[11pt,letterpaper]{article}
\usepackage{cogsys}
\usepackage[T1]{fontenc}
\usepackage{times}
\usepackage{fancyvrb}
\usepackage{tikz}
\usetikzlibrary{shapes,chains,calc,decorations.markings}

\usepackage{natbib}
\usepackage{amsmath}
\usepackage{amssymb}
\usepackage{amsthm}
\usepackage{multicol}
\usepackage{multirow}
\usepackage{forarray}

\cogsysheading{X}{20XX}{1-6}{X/20XX}{X/20XX}

\ShortHeadings{Inductive Learning of Answer Set Programs from Noisy Examples}
              {M.\ Law, A.\ Russo, and K.\ Broda}

\def\ruleend{.}
\newcommand{\asp}[1]{\mbox{$\mathtt{#1}$}}
\DeclareMathOperator{\codeif}{\mathtt{:-} }

\DeclareMathOperator{\naf}{\;\mathtt{not}\;}

\newcommand{\secref}[1]{Section~\ref{sec:#1}}

\newcommand{\thmref}[1]{Theorem~\ref{thm:#1}}
\newcommand{\tblref}[1]{Table~\ref{tbl:#1}}

\newcommand{\figref}[1]{Figure~\ref{fig:#1}}

\newcommand{\loasne}{ILP_{LOAS}^{noise}}

\newcommand{\answerset}[1]{
  \lbrace$\ForEach{;}{\ifnum\thislevelcount=1 \else ,$ $ \fi \asp{\thislevelitem}}{#1}$\rbrace
}

\theoremstyle{plain}
\newtheorem{theorem}{\protect\theoremname}[section]
\theoremstyle{plain}

\theoremstyle{definition}

\theoremstyle{definition}
\newtheorem{definition}{\protect\definitionname}[section]
\theoremstyle{definition}
\newtheorem*{defn*}{\protect\definitionname}
\theoremstyle{remark}
\newtheorem*{remark*}{\protect\remarkname}
\theoremstyle{remark}

\theoremstyle{plain}
\newtheorem*{example*}{\protect\examplename}
\theoremstyle{plain}

\theoremstyle{definition}
\newtheorem*{prop*}{\protect\propositionname}
\theoremstyle{definition}
\newtheorem{proposition}[theorem]{\protect\propositionname}
\theoremstyle{definition}
\newtheorem*{lemma*}{\protect\lemmaname}
\theoremstyle{definition}
\newtheorem*{cor*}{\protect\corname}
\theoremstyle{definition}

\providecommand{\definitionname}{Definition}
\providecommand{\metaprogramname}{Meta-program}
\providecommand{\propositionname}{Proposition}
\providecommand{\lemmaname}{Lemma}
\providecommand{\corname}{Corollary}
\providecommand{\examplename}{Example}
\providecommand{\remarkname}{Remark}
\providecommand{\theoremname}{Theorem}
\providecommand{\factname}{Fact}

\begin{document}

\title{Inductive Learning of Answer Set Programs from Noisy Examples}

\author{Mark Law}{mark.law09@imperial.ac.uk}
\author{Alessandra Russo}{a.russo@imperial.ac.uk}
\author{Krysia Broda}{k.broda@imperial.ac.uk}
\address{Department of Computing, Imperial College London, London, SW7 2AZ, United Kingdom}

\vskip 0.2in

\begin{abstract}
  In recent years, non-monotonic Inductive Logic Programming has received growing
interest. Specifically, several new learning frameworks and algorithms have
been introduced for learning under the answer set semantics, allowing the
learning of common-sense knowledge involving defaults and exceptions, which are
essential aspects of human reasoning.  In this paper, we present a
noise-tolerant generalisation of the \emph{learning from answer sets}
framework.  We evaluate our ILASP3 system, both on synthetic and on real
datasets, represented in the new framework. In particular, we show that on many
of the datasets ILASP3 achieves a higher accuracy than other ILP systems that
have previously been applied to the datasets, including a recently proposed
differentiable learning framework.

\end{abstract}

\section{Introduction}

The ultimate aim of cognitive systems is to achieve human-like intelligence. As
humans, we are capable of performing many cognitive activities such as learning
from past experience, predicting outcomes of our actions based on what we have
learned, and reasoning using our learned knowledge. Each of these cognitive
processes uses existing knowledge and generates new knowledge. They are
underpinned by our ability to perform {\em inductive reasoning}, one of our
most important high-level cognitive functions. Inductive reasoning is a complex
process by which new knowledge is inferred from a series of observations in a
way that can be {\em transferred} from past experiences to new situations. When
performing inductive reasoning, observations perceived through the environment
are often noisy and the existing knowledge that we use during the reasoning
process is also limited and incomplete. The human inductive reasoning process
is therefore capable of handling noise in the observations, reasoning with
incomplete and defeasible knowledge, applying knowledge learned in one scenario
to many other scenarios, and learning complex knowledge expressed in terms of
rules, constraints and preferences that can be communicated to others.

To realise cognitive systems able to perform human-like inductive reasoning,
Machine Learning (ML) solutions have to meet the above properties. Research in
ML has yielded approaches and systems that, although capable of identifying
patterns in datasets consisting of millions of (noisy) data points, cannot
express the learned knowledge in a form that could be understood by a human.
Moreover, their learned knowledge can only be used in exactly the scenario in
which it was learned: for example, a system trained to play Go on a standard
\texttt{19x19} board may not perform very well at Go played on a \texttt{20x20}
board. Lack of interpretability and transferability of the learned knowledge
make these approaches far from human cognition. On the other hand, Inductive
Logic Programming (ILP~\cite{Muggleton1991}) has been shown to be suited for
learning knowledge that can be understood by humans and applied to new
scenarios. Although approaches for performing ILP in the context of noisy
examples have been presented in the literature (e.g.\
\cite{sandewall1993handling,McCreath1997,hyper_n}), many existing ILP systems
can only learn knowledge expressed as definite logic programs, so they are not
capable of learning common-sense knowledge involving defaults and exceptions,
which are essential aspects of human reasoning. This type of knowledge can be
modelled using \emph{negation as failure}.

Recently, ILP has been extended to enable learning programs containing
negation as failure (e.g.\ \cite{ray2009nonmonotonic,Sakama2009}), and
interpreted under the answer set semantics~\citep{Gelfond1988}. In particular,
our recent results in inductive learning of answer
set programs (ILASP, \cite{JELIAILASP,ICLP16}) have demonstrated the ability
to support automated acquisition of complex knowledge structures in the
language of Answer Set Programming (ASP). The theoretical framework
underpinning ILASP, called \emph{Learning from Answer Sets} (LAS), enables the
learning of constraints, preferences and non-deterministic concepts. For
instance, LAS can learn the concept that a coin may non-deterministically land
on either heads or tails, but never both.

When learning, humans are also capable of disregarding information that does
not fit the general pattern. Any cognitive system that aims to mimic
human-level learning should therefore be capable of learning in the presence of
\emph{noisy} data. A realistic approach to cognitive knowledge acquisition is
therefore the learning of knowledge that covers the {\em majority} of the
examples, but which at the same time weights coverage against its
complexity. In this paper, we present a noise tolerant extension of our LAS
framework, \emph{Learning from noisy answer sets} ($ILP_{LOAS}^{noise}$) and
show that our ILASP3 system is capable of learning complex knowledge from noisy
data in an effective and scalable way. A collection of datasets, ranging from
synthetically generated to real datasets, is used to evaluate the performance
of the system with respect to the percentage of noise in the examples and to
compare it to existing ILP systems. Specifically, we consider two classes of
synthetically generated datasets, called Hamiltonian and Journey preferences,
and show that ILASP3 is able in both cases to achieve a high accuracy (of well
over 90\%), even with 20\% of the examples labelled incorrectly.  We also
evaluate ILASP3 on datasets concerning learning event theories~\citep{OLED},
sentence chunking~\citep{chunking_dataset}, preference
learning~\citep{kamishima2010survey,abbasnejad2013learning} and the synthetic
datasets of~\cite{evans2017learning}. Our results show that in most cases the
ability of ILASP3 to compute \emph{optimal} solutions for a given learning task
allows it to reach higher accuracy than the other systems, which do not
guarantee the computation of an optimal solution.

Next, in \secref{background}, we review relevant background material.
\secref{framework} introduces our new framework for learning ASP from noisy
examples; \secref{ilasp} discusses the ILASP algorithms;
Sections~\ref{sec:experiments} and~\ref{sec:rds} present an extensive
evaluation of our ILASP3 system; and finally, we conclude with a discussion of
related and future work.

\section{Background}
\label{sec:background}

We briefly introduce basic notions and terminologies used throughout the paper.
Given any atoms $\asp{h, h_1,\ldots,h_k, b_1,\ldots,b_n,c_1,\ldots,c_m}$, a
{\em normal rule} is of the form $\asp{h \codeif b_1,\ldots, b_n, \naf c_{1},\ldots,}$
$\asp{\naf c_{m}}$, where ``$\mathtt{not}$'' is negation as failure, $h$ is the
head of the rule and $\asp{b_1,\ldots, b_n, \naf c_{1},\ldots, \naf c_{m}}$ is
the body of the rule.  For example, $\asp{fly(X) \codeif bird(X),\naf ab(X)}$
is a normal rule stating that any bird can fly, unless it is abnormal. The
negated condition $\asp{\naf ab(X)}$ is assumed to hold unless there is a way
of proving $\asp{ab(X)}$ for some value of $X$. So, the normal rule essentially
models that by default, birds can fly, unless there is a proof that the bird is
abnormal. ASP programs include three other types of rule: choice rules, hard
and weak constraints. A \emph{choice rule} is of the form $\asp{l\{h_{1},
\ldots, h_{k}\}u\codeif b_1,\ldots, b_n,}$ $\asp{\naf c_{1},\ldots, \naf
c_{m}}$, where $\asp{l}$ and $\asp{u}$ are integers and $\asp{l\{h_{1}, \ldots,
h_{k}\}u}$ is called an \emph{aggregate}.  A \emph{hard constraint} is of the
form $\asp{\codeif b_1,\ldots, b_n,\naf c_{1},\ldots,\naf c_{m}}$ and a
\emph{weak constraint} is of the form $\asp{:\sim b_1,\ldots,b_n, \naf
c_1,\ldots,}$ $\asp{\naf c_m\ruleend[w@l,t_1,\ldots, t_k]}$ where $\asp{w}$ and
$\asp{l}$ are terms specifying \emph{weight} and {\em priority level}, and
$\asp{t_1,\ldots,t_k}$ are terms.

The \emph{Herbrand Base} of an ASP program $P$, denoted $HB_P$, is the set of
ground (variable free) atoms that can be formed from predicates and constants
in $P$. Subsets of $HB_P$ are called (Herbrand) interpretations of $P$. The
semantics of ASP programs $P$ are defined in terms of \emph{answer sets} -- a
special\footnote{For a formal definition of answer sets of the programs in this
paper see~\cite{simpleReduct}.} subset of interpretations of $P$, denoted as
$AS(P)$, that satisfy every rule in $P$.  Given an answer set $A$, a ground
normal or choice rule is satisfied if the head is satisfied by $A$ whenever all
positive atoms and none of the negated atoms of the body are in $A$, that is
when the body is satisfied. A ground aggregate $\asp{l\{h_{1}, \ldots,
h_{k}\}u}$ is satisfied by an interpretation $I$ iff $\asp{l}\leq|
I\cap\{\asp{h_{1}, \ldots,h_{k}}\}|\leq \asp{u}$. So, informally, a ground
choice rule is satisfied by an answer set $A$ if whenever its body is satisfied
by an answer set $A$, a number between $\asp{l}$ and $\asp{u}$ (inclusive) of
the atoms in the aggregate are also in $A$. A ground constraint is satisfied
when its body is not satisfied. A constraint therefore has the effect of
eliminating all answer sets that satisfy its body. Weak constraints do not
affect what is, or is not, an answer set of a program $P$.  Instead, they
create an ordering $\succ_{P}$ over $AS(P)$ specifying which answer sets are
``preferred'' to others.  Informally, at each \emph{priority level} $\asp{l}$,
satisfying weak constraints with level $l$ means discarding any answer set that
does not minimise the sum of the weights of the ground weak constraints (with
level $\asp{l}$) whose bodies are satisfied. Higher levels are minimised first.
For example, the two weak constraints $\asp{:\sim mode(L, walk), distance(L,
D).[D@2, L]}$ and $\asp{:\sim cost(L, C).[C@1, L]}$ express a preference
ordering over alternative journeys. The first constraint (at priority 2)
expresses that the total walking distance (the sum of the distances of journey
legs whose mode of transport is $\asp{walk}$) should be minimised, and the
second constraint expresses that the total cost of the journey should be
minimised. As the first weak constraint has a higher priority level than the
second, it is minimised first -- so given a journey $j_1$ with a higher cost
than another journey $j_2$, $j_1$ is still preferred to $j_2$ so long as the
walking distance of $j_1$ is lower than that of $j_2$. The set $ord(P)$
captures the ordering of interpretations induced by $P$ and generalises the
$\succ_{P}$ relation, so it not only includes $\langle A_1, A_2, <\rangle$ if
$A_1 \succ_{P} A_2$, but includes tuples for each binary comparison operator
($<$, $>$, $=$, $\leq$, $\geq$ and $\neq$).

 A \emph{partial interpretation}, $e_{pi}$, is a pair of sets of ground atoms
$\langle e^{inc}, e_{pi}^{exc}\rangle$. An interpretation $I$ \emph{extends}
$e_{pi}$ iff $e_{pi}^{inc} \subseteq I$ and $e_{pi}^{exc} \cap I = \emptyset$.
Examples for learning come in two forms: \emph{context-dependent partial
interpretations} (CDPIs) and \emph{context-dependent ordering examples}
(CDOEs).
A CDPI example $e$ is a pair $\langle e_{pi}, e_{ctx}\rangle$, where
$e_{pi}$ is a partial interpretation and $e_{ctx}$ is a program with no weak
constraints called the context of $e$.  A program $P$ is said to \emph{bravely
accept} $e$ if there is at least one answer set $A$ of $P \cup e_{ctx}$ that
extends $e_{pi}$ -- such an $A$ is called an accepting answer set of $P$ wrt
$e$.  Essentially, a CDPI says that the learned program, together with the
context of $e$, should bravely\footnote{A program $P$ \emph{bravely entails}
an atom $\asp{a}$ if there is at least one answer set of $P$ that contains
$a$.} entail all inclusion atoms and none of the exclusion atoms of $e$. CDPIs
can be used for \emph{classification} tasks, as they specify that given
contexts should entail given conjunctions of atoms. But as learned programs may
have multiple answer sets, accepting a CDPI may require additional assumptions
to be made.
A CDOE $o$ is a tuple $\langle e_1, e_2, op\rangle$, where the first two
elements are CDPIs and $op$ is a binary comparison operator. A program $P$ is
said to \emph{bravely respect} $o$ if there is a pair of accepting answer sets,
$A_1$ and $A_2$, of $P$ wrt $e_1$ and $e_2$, respectively, such that $\langle
A_1, A_2, op\rangle \in ord(P)$.  $P$ is said to \emph{cautiously respect} $o$
if for every pair, $A_1$ and $A_2$, of accepting answer sets of $P$ (wrt $e_1$
and $e_2$, respectively), $\langle A_1, A_2, op\rangle \in ord(P)$.  CDOEs
enable \emph{preference learning} as they specify which answer sets should be
prefered to other answer sets.

An $ILP_{LOAS}^{context}$ task $T$ consists of an ASP background knowledge $B$,
a hypothesis space $S_M$, labelled CDPIs, $E^{+}$ (positive examples) and
$E^{-}$  (negative examples), and labelled CDOEs, $O^{b}$ (brave orderings) and
$O^{c}$ (cautious orderings). $S_M$ is the set of rules allowed in hypotheses.
A hypothesis $H\subseteq S_M$ \emph{covers} a positive (resp. negative) example
$e$ if $B\cup H$ accepts (resp. does not accept) $e$. $H$ covers a brave (resp.
cautious) ordering $o$ if $B\cup H$ bravely (resp. cautiously) respects $o$.
$H$ is an inductive solution of $T$ iff $H$ covers every example in $T$.

\section{Learning Framework}
\label{sec:framework}

This section presents the $ILP_{LOAS}^{noise}$ framework, which extends our previous (non-noisy) learning framework
$ILP_{LOAS}^{context}$ (\cite{ICLP16}), by allowing examples to be \emph{weighted
context-dependent partial interpretations} and \emph{weighted context-dependent
ordering examples}.  These are essentially the same as CDPIs and CDOEs, but 
weighted with a notion of \emph{penalty}. If a hypothesis does not
cover an example, we say that it \emph{pays the penalty} of that example.
Informally, penalties are used to calculate the \emph{cost} associated with a
hypothesis for not covering examples. The cost function of a hypothesis $H$ is
the sum over the penalties of all of the examples that are not \emph{covered}
by $H$, augmented with the length of the hypothesis. The goal of
$ILP_{LOAS}^{noise}$ is to find a hypothesis that minimises the cost function
over a given hypothesis space with respect to a given set of examples.

\begin{definition}
  A \emph{weighted context-dependent partial interpretation} $e$ is a tuple
  $\langle e_{id}, e_{pen}, e_{cdpi}\rangle$, where $e_{id}$ is a constant,
  called the \emph{identifier} of $e$ (unique to each example), $e_{pen}$ is
  the penalty of $e$ and $e_{cdpi}$ is a context-dependent partial
  interpretation. The penalty $e_{pen}$ is either a positive integer, or
  $\infty$. A program $P$ \emph{accepts} $e$ iff it accepts $e_{cdpi}$.
%
  A \emph{weighted context-dependent ordering example} $o$ is a tuple $\langle
  o_{id}, o_{pen}, o_{ord}\rangle$, where $o_{id}$ is a constant, called the
  \emph{identifier} of $o$, $o_{pen}$ is the penalty of $o$ and $o_{ord}$ is a
  CDOE. The penalty $o_{pen}$ is either a positive integer, or $\infty$. A
  program $P$ \emph{bravely} (resp.\ \emph{cautiously}) \emph{respects} $o$ iff
  it bravely (resp.\ cautiously) respects $o_{ord}$.
\end{definition}

In learning tasks without noise, each example must be covered by any inductive
solution. However, when examples are noisy (i.e.\ they have a weight),
inductive solutions need not cover every example, but they incur penalties for
each uncovered example.
Multiple occurrences of the same CDPI example have different identifiers. So
hypotheses that do not
cover that example will pay the penalty multiple times (for instance, if a
CDPI occurs twice then a hypothesis will have to pay twice the penalty for not
covering it).
In most of the learning tasks presented in this paper, all examples have the
same penalty. In some cases, however, penalties are used to simulate
\emph{oversampling}; for example, in tasks with far more positive examples than
negative examples, we may choose to give the negative examples a higher weight
-- otherwise it is likely that the learned hypothesis will treat all negative
examples as noisy.

Our learning task with noisy examples consists of an ASP background knowledge,
weighted CDPI and CDOE examples and a hypothesis space,\footnote{For details of
hypothesis spaces in this paper, see
\url{https://www.doc.ic.ac.uk/~ml1909/ILASP/}.} which defines the
set of rules allowed to be used in constructing solutions of the task. These
tasks are \emph{supervised} learning tasks, as all examples are
labelled, as positive/negative, or with an operator in the case of the ordering
examples.

\begin{definition}\label{def:lnas}
  An $\loasne$ task $T$ is a tuple of the form $\langle B, S_M, \langle E^+,
  E^-, O^b, O^c\rangle\rangle$, where $B$ is an ASP program, $S_M$ is a
  hypothesis space, $E^+$ and $E^-$ are sets of weighted CDPIs and $O^b$ and
  $O^c$ are sets of weighted CDOEs. Given a hypothesis $H \subseteq S_M$,

  \begin{enumerate}
    \item $uncov(H, T)$ is the set consisting of
      all examples $e \in E^+$ (resp. $E^-$) such that $B \cup H$ does not
      accept (resp. accepts) $e$ and
      all ordering examples $o \in O^b$ (resp. $O^c$) such that $B \cup H$ does
      not bravely (resp. cautiously) respect $o$.
    \item
      the penalty of $H$, denoted as $pen(H, T)$, is the sum $\sum_{e \in
      uncov(H, T)} e_{pen}$.
    \item
      the score of $H$, denoted as $\mathcal{S}(H, T)$, is the sum $|H| + pen(H, T)$.
    \item $H$ is an inductive solution of $T$ (written $H \in \loasne(T)$) if
      and only if $\mathcal{S}(H, T)$ is finite.
    \item $H$ is an \emph{optimal inductive solution} of $T$ (written $H \in$
      $^*\loasne(T)$) if and only if $\mathcal{S}(H, T)$ is finite and
      $\nexists H' \subseteq S_M$ such that $\mathcal{S}(H, T) >
      \mathcal{S}(H', T)$.
  \end{enumerate}

\end{definition}

Examples with infinite penalty \emph{must} be covered by any
inductive solution, as any hypothesis that does not cover such an
example will have an infinite score.  An $\loasne$ task $T$ is said to be
\emph{satisfiable} if $\loasne(T)$ is non-empty. If $\loasne(T)$ is empty, then
$T$ is said to be \emph{unsatisfiable}.
Theorem~\ref{thm:complexity} shows that for propositional tasks (where all
hypothesis spaces, contexts and background knowledge are propositional) the
complexity of $ILP_{LOAS}^{noise}$ is the same as $ILP_{LOAS}^{context}$ for
the decision problems of \emph{verification} -- deciding if a given
hypothesis is a solution of a given task -- and \emph{satisfiability} --
deciding if a given task has any solutions -- investigated
in~\cite{AIJ17}.

\begin{theorem}\label{thm:complexity} $ $

\begin{enumerate}
  \item Deciding verification for an arbitrary propositional $ILP_{LOAS}^{noise}$ task is $DP$-complete
  \item Deciding satisfiability for an arbitrary propositional $ILP_{LOAS}^{noise}$ task is $\Sigma^P_2$-complete
\end{enumerate}
\end{theorem}

Like its predecessor $ILP_{LOAS}^{context}$, our new learning framework
$ILP_{LOAS}^{noise}$ for noisy examples
is capable of learning complex human-interpretable
knowledge, containing defaults, non-determinism, exceptions and preferences.
The generalisation to allow penalties on the examples means that the new
framework can be deployed in realistic settings where examples are not
guaranteed to be correctly labelled. \thmref{complexity} shows that this
generalisation does not come at any additional cost in terms of the
computational complexity of important decision problems of the framework.

\section{The ILASP system}\label{sec:ilasp}

ILASP (Inductive Learning of Answer Set Programs,
\cite{JELIAILASP,ILASP_system,ICLP15,ICLP16}) is a collection of algorithms for
solving LAS tasks. The general idea behind the ILASP approach is to transform a
learning task into a meta-level ASP program, which can be iteratively solved
(extending the program in each iteration) until the optimal answer sets of the
program correspond to solutions of the learning task. Unlike many
other ILP systems, such as
\cite{muggleton1995inverse,ray2009nonmonotonic,inspire}, the ILASP algorithms
are guaranteed to return an optimal solution of the input learning task (with
respect to the cost function). This can of course mean that ILASP may take longer
to compute a solution than \emph{approximate} systems (which are not guaranteed
to return an optimal solution); however, as we demonstrate in \secref{rds}, the
hypotheses found by ILASP are often more accurate than those found by
approximate systems.

Each version of ILASP has aimed to address scalability issues of the previous
versions.\break ILASP1~\citep{JELIAILASP} was a prototype implementation, with a major
efficiency issue with respect to negative examples. ILASP2~\citep{ICLP15}
addressed this issue by introducing a notion of an \emph{violating reason}.
In each iteration,  each answer set of the ILASP2 meta-level program $T_{meta}$ contains a
representation of a hypothesis which covers every positive example and every
brave ordering example. An answer set representing a hypothesis that
is not an inductive solution, contains a ``reason'' why at least one
negative example or cautious ordering is not covered, which can be translated
into an ASP representation that, when added to $T_{meta}$, rules out any
hypothesis that is not a solution for this reason. This process is performed
iteratively until no more violating reasons are detected. For full details of
violating reasons, see~\cite{ICLP15}.

Both ILASP1 and ILASP2 scale poorly with respect to the number of examples, as
the number of rules in the ground instantiation of their meta-level
representation is proportional to the number of examples in the learning task.
As many examples may be similar, and thus covered by the same hypotheses, in
non-noisy tasks (where all examples must be covered), it is often sufficient to
consider a small subset of the examples called a \emph{relevant subset} of the
examples. ILASP2i~\citep{ICLP16} uses this property to further improve the
scalability of ILASP2.  It starts with an empty set of \emph{relevant examples}
$RE$, and, at each iteration, it calls ILASP2 on a learning task using only the
examples in $RE$. The hypothesis returned by ILASP2 is guaranteed to cover the
current relevant examples, but is not necessarily an inductive solution of the
original task. So, if ILASP2 returns a hypothesis that does not cover at least
one example, then an arbitrary uncovered example is added to $RE$ and the next
iteration is started. If no such example exists, then the hypothesis is
returned as an optimal inductive solution of the original task.
\cite{ICLP16} showed that ILASP2i can be up to two orders of magnitude faster
than ILASP2 on tasks with 500 (noise-free) examples.

Both ILASP2 and ILASP2i can be extended to solve $ILP_{LOAS}^{noise}$ tasks;
however, neither algorithm is well suited to solving tasks with a large number
of noise examples with finite penalties. ILASP2 does not scale with respect to the
number of examples (regardless of whether examples have finite penalties), and
the relevant example feature of ILASP2i is not equally effective when examples
have penalties. One reason for this is that many noisy examples may have to be
added to the relevant example set before the cost of not covering a particular
class of relevant examples is enough to outweigh the cost of learning an extra
rule in the hypothesis. The most recent ILASP algorithm, ILASP3, iteratively
translates examples into \emph{hypothesis constraints} -- constraints on the
structure of a hypothesis that are satisfied if and only if the hypothesis
covers the example. This leads to a much more compact meta-level program,
defined in terms of these hypothesis constraints. Once hypothesis constraints
have been computed for one example $e$, it is possible to compute the set of
other examples (which have not yet been translated into hypothesis constraints)
that are definitely not covered if $e$ is not covered. This means that one
relevant example can effectively have a much higher penalty than just the
penalty for that example, meaning that the number of relevant examples that are
needed in ILASP3 is often lower than those needed by ILASP2i.

\section{Evaluation of ILASP3 on synthetic datasets}
\label{sec:experiments}

In this section ILASP3 is evaluated on two synthetic datasets, the first of
which is aimed at learning normal rules, choice rules and hard constraints,
while the second is aimed at learning weak constraints. The value of using
synthetic datasets is that we can control the amount of noise and investigate
how the accuracy and running time of ILASP3 varies with the amount of noise.

\subsection{Hamilton Graphs}
\label{sec:hamilton}
\begin{figure*}[t]

    \hfill
        \includegraphics[width=0.40\textwidth, trim={15mm 0mm 17mm 0mm},clip]{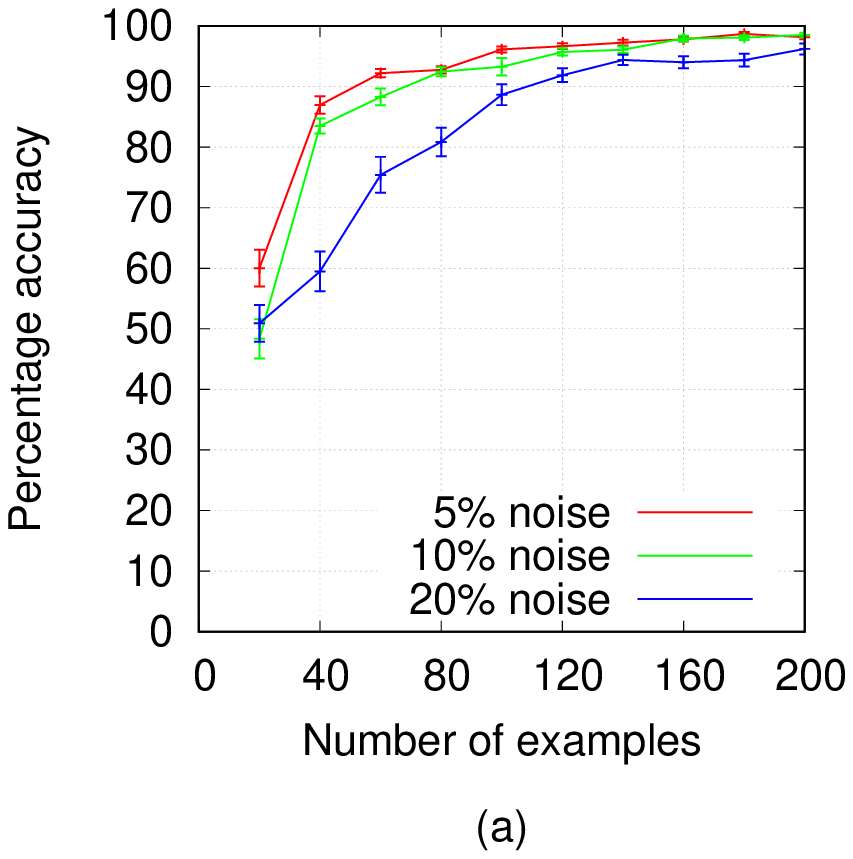}
        \hfill
        \includegraphics[width=0.40\textwidth, trim={15mm 0mm 17mm 0mm},clip]{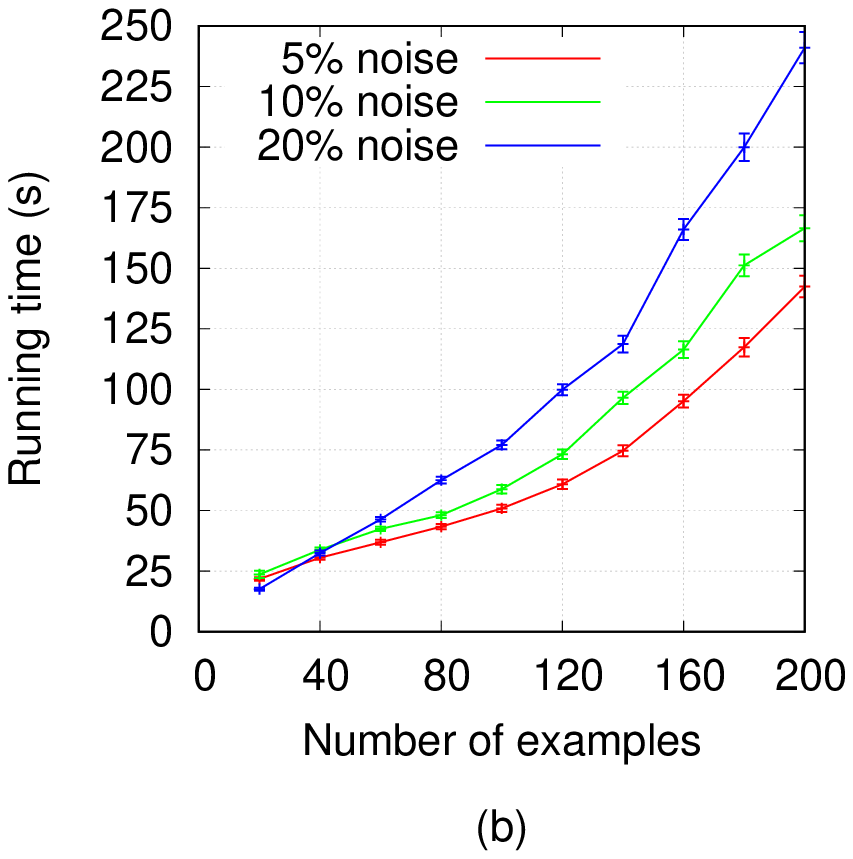}
        \hfill
        \!

    \caption{\label{fig:hamilton} (a) the average computation time and (b)
    average accuracy of ILASP3 for the Hamilton learning
    task, with varying numbers of examples, and varying noise.}

\end{figure*}

In this experiment the task is to learn the definition of what it means for a
graph to be Hamiltonian. This concept was chosen as it requires learning a
hypothesis that contains choice rules, recursive rules and hard constraints,
and also contains negation as failure. In these experiments, we show
that ILASP3 could learn this hypothesis in the presence of noise, and we test
how the running time of ILASP3 is affected by the number of examples and the
number of incorrectly labeled examples.

For $n = 20, 40, \ldots, 200$, $n$ random graphs of size one to four were
generated, half of which were Hamiltonian. The graphs were labelled as either
positive or negative, where positive indicates that the graph is Hamiltonian.
The correct ASP representation of Hamiltonian and a discussion of the
representation of examples in this task is given in
Appendix~\ref{sec:representations}.

We ran three sets of experiments to evaluate ILASP3 on the Hamilton learning
problem, with 5\%, 10\% and 20\% of the examples being labelled incorrectly. In
each experiment, an equal number of Hamiltonian graphs and non-Hamiltonian
graphs were randomly generated and 5\%, 10\% or 20\% of the examples were
chosen at random to be labelled incorrectly. This set of examples were labelled
as positive (resp. negative) if the graph was not (resp. was) Hamiltonian. The
remaining examples were labelled correctly (positive if the graph was
Hamiltonian; negative if the graph was not Hamiltonian). \figref{hamilton}
shows the average accuracy and running time of ILASP3 with up to 200 example
graphs. Each experiment was repeated 50 times (with different randomly
generated examples). In each case, the accuracy was tested by generating a
further 1,000 graphs and using the learned hypothesis to classify the graphs as
either Hamiltonian or non-Hamiltonian (based on whether the hypothesis was
satisfiable when combined with the representation of the graph).

The experiments show that on average ILASP3 is able to achieve a high accuracy
(of well over 90\%), even with 20\% of the examples labelled incorrectly. A
larger percentage of noise means that ILASP3 requires a larger number of
examples to achieve a high accuracy. This is to be expected, as with few
examples, the hypothesis is more likely to ``overfit'' to the noise, or pay the
penalty of some non-noisy examples. With large numbers of examples, it is more
likely that ignoring some non-noisy examples would mean not covering others,
and thus paying a larger penalty. The computation time rises in all three
graphs as the number of examples increases. This is because larger numbers of
examples are likely to require larger numbers of iterations of the ILASP3
algorithm. Similarly, more noise is also likely to mean a larger number of
iterations.

\subsection{Noisy Journey Preferences}
\label{sec:jp}

The experiment in this section is a noisy extension of the journey preference
learning setting used in \cite{ICLP16}, where the goal is to learn a user's
preferences from a set of ordered pairs of journeys. These experiments aim to
show that ILASP3 is capable of preference learning in the presence of noise,
and to test how the accuracy and running time of ILASP3 are affected by the
numbers of examples and the proportion of examples which are incorrectly
labelled.

In each experiment, we selected a ``target hypothesis'' consisting of between
one and three weak constraints from a hypothesis space of weak constraints
(discussed in Appendix~\ref{sec:representations}). For each set of weak
constraints, we then ran learning tasks with 0, 20, $\ldots$, 200 examples and
with $5\%$, $10\%$ and $20\%$ noise. The ordering examples for these learning
tasks were generated from the weak constraints such that half of the (brave)
ordering examples represented pairs of journeys $J_1$ and $J_2$ where $J_1$ was
strictly preferred to $J_2$, given the weak constraints, and the other half
represented journeys such that $J_1$ was equally preferred to $J_2$. Depending
on the level of noise, either $5\%$, $10\%$ or $20\%$ of the examples were
given with the wrong operator ($>$ instead of $<$ and $\neq$ instead of $=$).
Each ordering example was given a penalty of one.

\begin{figure*}[t]

    \hfill
    \includegraphics[width=0.40\textwidth, trim={18mm 0mm 15mm 0mm},clip]{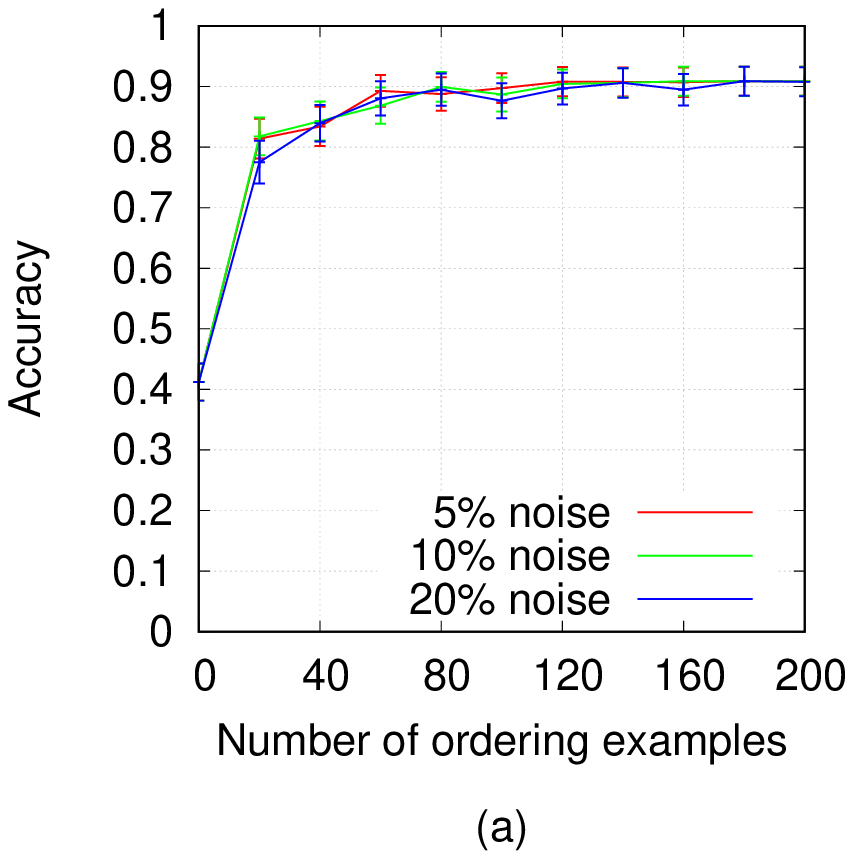}
    \hfill
    \includegraphics[width=0.40\textwidth, trim={18mm 0mm 15mm 0mm},clip]{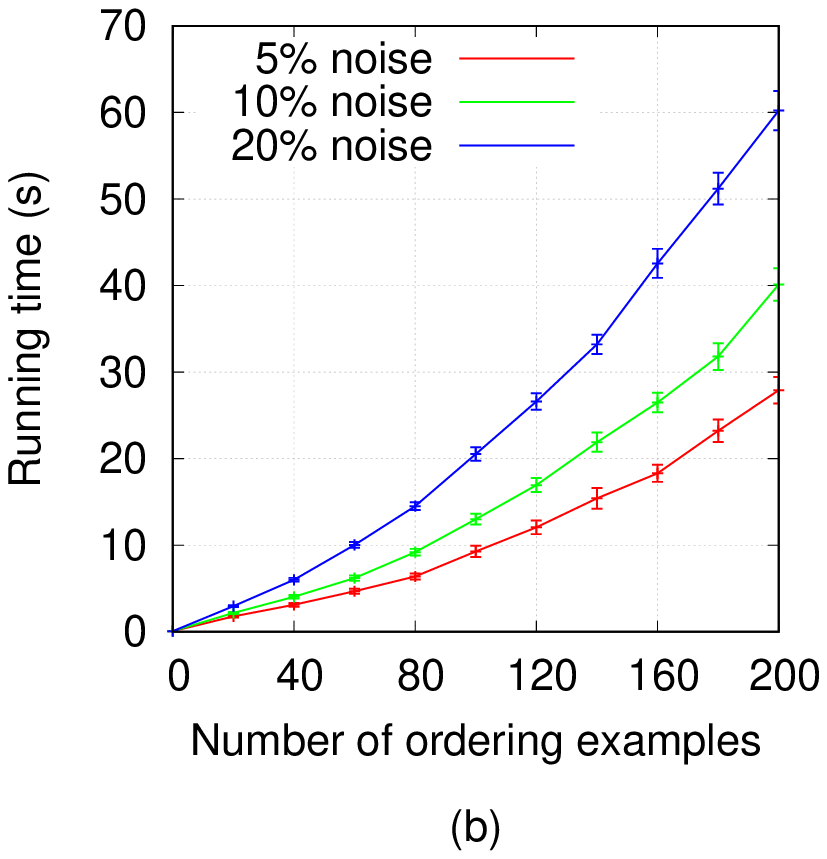}
    \hfill
    \!

    \hfill
    \includegraphics[width=0.40\textwidth, trim={18mm 0mm 15mm 0mm},clip]{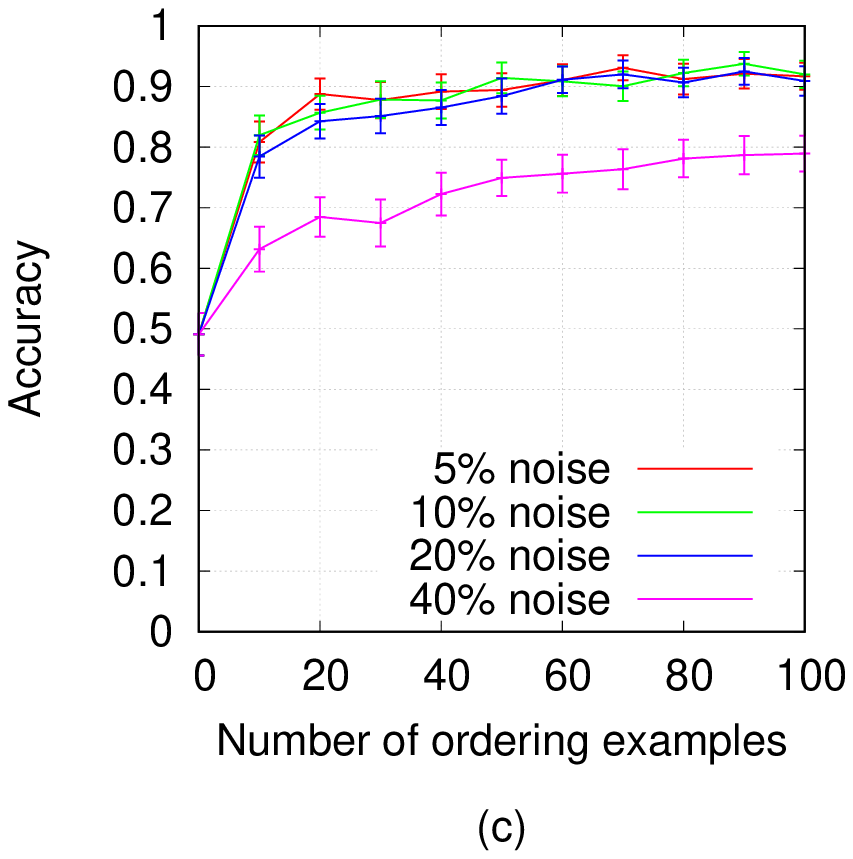}
    \hfill
    \!

    \caption{\label{fig:jpn} (a) and (c) the average accuracy and (b) average
    computation time of ILASP3 for the journey preference learning task, with
    varying numbers of examples, and varying noise. Each point in the graphs is
    an average over 50 different tasks.}

\end{figure*}

The results (\figref{jpn} (a)) show that even with $20\%$ noise, ILASP3 was
able to learn hypotheses with an average accuracy of over 90\%.  There was not
much difference between ILASP3's accuracy with 5\%, 10\% and 20\% noise,
although the noisier tasks had a higher computation time (this is shown in
\figref{jpn} (b)), as in general ILASP3 requires more iterations on noisier
tasks. Even with 20\% noise and 200 ordering examples, ILASP3 terminated in
just over 60 seconds on average.

As the results for 20\% noise were so close to the results for 5\% noise, we
ran a further set of examples to check that there was some limit to the level
of noise where ILASP3 would no longer learn such an accurate
hypothesis.\footnote{If ILASP could achieve such a high accuracy, even with very
high levels of noise, then this would indicate that the hypothesis space was
too restrictive, and it was impossible to learn anything other than an accurate
hypothesis.} In this second set of experiments, we tested ILASP3 with up to
40\% noise, and investigated with 0, 10, $\ldots$, 100 examples. With 40\%
noise, the accuracy was lower, but ILASP still achieved an average accuracy of
just under 80\%.

These experiments show that ILASP3 is able to accurately learn a set of weak
constraints from examples of the orderings of answer sets given by these weak
constraints, even when 20\% of the orderings are incorrect.  Although the
running time of ILASP3 is affected by the number of examples and the proportion
of incorrectly labelled examples, ILASP3 is able to find an optimal solution in
an average of 60 seconds, even with 200 ordering examples, 20\% of which are
incorrectly labelled.
Learning weak constraints is significant, as they can be
used to represent user preferences. In
Sections~\ref{sec:car}~and~\ref{sec:sushi}, we apply ILASP3 on two real
preference learning datasets.

\section{Comparison with Other Systems}\label{sec:rds}

The experiments in this section use datasets that have previously been used to
evaluate other ILP systems in the presence of noise. Unlike ILASP3, none of the
systems we compare with aim to find optimal solutions. The aim of this set of
experiments is therefore to test whether finding optimal solutions leads to any
gain in accuracy over systems which may return sub-optimal solutions.


\subsection{CAVIAR Dataset}
\label{sec:caviar}

In this experiment ILASP3 was tested on the recent CAVIAR dataset that has been
used to evaluate the OLED~\citep{OLED} system, which is an extension of the
XHAIL~\citep{ray2009nonmonotonic} algorithm, for learning Event
Calculus~\citep{kowalski1986logic} theories. The dataset contains data gathered
from a video stream. Information such as the positions of people has been
extracted from the stream, and humans have annotated the data to specify when
any two people are interacting. Specifically, we consider a task
from~\cite{OLED}, in which the aim is to learn rules to define initiating and
terminating conditions for two people meeting.  In the evaluation of the OLED
system, examples were generated for every pair of consecutive timepoints $t$
and $t+1$. Each example is a pair $\langle N\cup A_t, A_{t+1}\rangle$, where
$N$ is the ``narrative'' at time $t$ (a collection of information about the
people in the video stream, such as their location and direction), and $A_i$ is
the ``annotation'' at time $i$ (exactly which pairs of people in the video have
been labelled as meeting). This is very simple to express using
context-dependent examples. The context of an example is simply the narrative
and annotation of time $t$ together with a set of constraints that enforce that
the meetings at time $t$ are exactly those in the annotation. The aim of this
experiment is to compare ILASP3 to OLED, which was specifically designed to
solve this kind of task efficiently. We aimed to discover whether ILASP3 is
able to find better quality hypotheses than OLED (in terms of the $F_1$ measure
used to evaluate the hypotheses found by OLED), and whether ILASP3's guarantee
of finding an optimal solution comes at a cost in terms of running time.

In total there are 24,530 consecutive pairs in the dataset.\footnote{We used the
data from \url{users.iit.demokritos.gr/~nkatz/OLED-data/caviar.json.tar.gz}}
We performed ten-fold cross validation by randomly partitioning the dataset. As
there were only twenty-two timepoints where the group of people meeting was
different to the timepoint before, these examples were given a high penalty (of
100).  Effectively this is the same as oversampling this class of examples. If
all examples had been given a penalty of one, then ILASP3 would have likely
learned the empty hypothesis, as the twenty-two examples in a task of
many thousands of examples would likely be treated as noise.

We compare ILASP3 to OLED on the measures of precision, recall and the $F_1$
score.\footnote{Let $tp$, $tn$, $fp$, $fn$ represent the number of true
positives, true negatives, false positives and false negatives achieved by a
classifier on some test data.  The \emph{precision} of the classifier (on this
test data) is equal to $\frac{tp}{tp + fp}$ and the \emph{recall} is equal to
$\frac{tp}{tp + fn}$.  The $F_1$ score is equal to $\frac{2 \times precision
\times recall}{precision + recall}$.}  ILASP3 achieved a precision of 0.832 and a
recall of 0.853, giving an $F_1$ score of 0.842, compared with OLED's precision
of 0.678 and recall of 0.953, with an average $F_1$ score of 0.792. ILASP3's
average running time was significantly higher at 576.3s compared with OLED's
107s. This is explained by the fact that the OLED system computes hypotheses
through theory revision, iteratively processing examples in sequence to
continuously revise its hypothesis. This means that, unlike ILASP3, OLED is not
guaranteed to find an optimal solution of a learning task.

We note several key differences between our experiments and those reported
in~\cite{OLED}. Firstly, to reduce the number of irrelevant answer
sets (which lead to slow computation), we constrained the hypothesis
space stating that rules for $\asp{terminatedAt(meeting(V1, V2), T)}$ had
to contain $\asp{holdsAt(meeting(V1, V2), T)}$ in the body, which ensures that
a fluent can only be terminated if it is currently happening. Similarly, any
rule for $\asp{initiatedAt(meeting(V1, V2), T)}$ had to contain $\asp{\naf
holdsAt(meeting(V1, V2), T)}$ in the body. OLED does not employ this
constraint, but when processing an example pair of time points, only considers
learning a new rule for $\asp{initiatedAt}$, for example, if two people are
meeting at time $t+1$, but not at time $t$.
The second difference in our experiment is that ILASP3 enumerates the
hypothesis space in full. As the hypothesis space in this task is potentially
very large, several ``common sense'' constraints were enforced on the
hypothesis space; for instance, two people cannot be both close to and far away
from each other at the same time (rules with both conditions in the body were
not generated).  In total, the hypothesis space contained 3,370 rules.  OLED
does not enumerate the hypothesis space in full, but uses an approach similar
to XHAIL, and derives a ``bottom clause'' from the background knowledge and the
example. In most cases (unless there is noise in the narrative, suggesting that
two people are both close to and far away from each other) OLED will therefore
only consider rules that respect the ``common sense'' constraints, as other
rules would not be derivable.

This experiment has shown that, at least on this dataset, ILASP3's guarantee of
finding an optimal solution can lead to better quality hypotheses than those
found by OLED; however, this quality comes at a cost, as ILASP3's running time
is significantly higher than OLED's.

\subsection{Sentence Chunking}
\label{sec:sentence}

In~\cite{inspire}, the Inspire system was evaluated on a sentence
chunking~\citep{chunking} dataset~\citep{chunking_dataset}. The task in this
setting is to learn to split a sentence into short phrases called chunks. For
instance, according to the dataset~\citep{chunking_dataset}, the sentence ``Thai
opposition party to boycott general election.'' should be split into the three
chunks ``Thai opposition party'', ``to boycott'' and ``general election''.
\cite{inspire} describe how to transform each sentence into a set of facts
consisting of part of speach (POS) tags. We use each of these sets of facts as
the context of a context dependent example. In Inspire (which is a brave
induction system), the facts are all put into the background knowledge. The
task is to learn a predicate $\asp{split/1}$, which expresses where sentences
should be split. Inspire does not guarantee finding an optimal solution. The
hypothesis can be suboptimal for three reasons: firstly, the abductive phase
may find an abductive solution which leads to a suboptimal inductive solution;
secondly, Inspire's pruning may remove some hypotheses from the hypothesis
space; and finally, Inspire was set to interrupt the inductive phase after
1,800 seconds, returning the most optimal hypothesis found so far. In these
experiments, we aimed to show that ILASP3's guarantee of finding an optimal
solution leads to a better quality hypotheses than Inspire's approximations,
and if so, whether ILASP3's running time was higher Inspire's timeout of 1,800s.

Note that the Inspire tasks in~\cite{inspire} group the multiple $\asp{split}$
examples for a chunk into a single example (using a $\asp{goodchunk}$
predicate); for example, the background knowledge may contain a rule
$\asp{goodchunk(1) \codeif split(1), \naf split(2), \naf split(3), split(4)}$
expressing that there is a chunk between words one and four of a sentence. It
is noted in~\cite{inspire} that this increased performance. This is because
there is no benefit in covering some of the $\asp{split}$ atoms that make up a
chunk, as hypotheses are tested over full chunks rather than splits.  In our
framework, we represent this directly with no need for the $\asp{goodchunk}$
rules, with the individual split atoms being inclusions and exclusions in the
partial interpretation of the example and the penalty being on the full
example. In our learning task, the example corresponding to the rule for
$\asp{goodchunk(1)}$ would have the partial interpretation
$\langle\answerset{split(1);split(4)},\answerset{split(2);split(3)}\rangle$.
In~\cite{inspire}, eleven-fold cross validation was performed on five different
datasets, with 100 and 500 examples. As Inspire has a parameter which
determines how aggressive the pruning should be, \cite{inspire} present several
$F_1$ scores, for different values of this parameter. Each entry for Inspire in
Table~\ref{tbl:inspire} is Inspire's best $F_1$ score over all pruning
parameters.

\begin{table}[t]
  \vskip -0.15in

  \caption{\label{tbl:inspire} $F_1$ scores for Inspire and
  ILASP3 and ILASP3's average running time on the sentence chunking tasks.
  }

  \vskip 0.1in

  \small

  \begin{center}
  \begin{tabular}{|c|c|c|c|c|}
    \hline
    & & Inspire $F_1$ score & ILASP $F_1$ score & ILASP time (s)\\\hline
    \multirow{5}{*}{100 examples}
    &   Headlines S1   & 73.1 & 74.2 & 351.2 \\
    &   Headlines S2   & 70.7 & 73.0 & 388.3 \\
    &   Images S1      & 81.8 & 83.0 & 144.9 \\
    &   Images S2      & 73.9 & 75.2 & 187.2 \\
    &   Students S1/S2 & 67.0 & 72.5 & 264.5 \\\hline
    \multirow{5}{*}{500 examples}
    &   Headlines S1   & 69.7 & 75.3 & 1,616.6 \\
    &   Headlines S2   & 73.4 & 77.2 & 1,563.6 \\
    &   Images S1      & 75.3 & 80.8 & 929.8 \\
    &   Images S2      & 71.3 & 78.9 & 935.8 \\
    &   Students S1/S2 & 66.3 & 75.6 & 1,451.3 \\\hline
  \end{tabular}
  \end{center}

  \vskip -0.1in

\end{table}

Inspire approximates the optimal inductive solution of the task and has
a timeout of 1,800s on the inductive phase -- in contrast, ILASP3 terminated in
less than 1,800 seconds on every task. ILASP3 achieved a higher average $F_1$
score than Inspire on every one of the ten tasks. This shows that computing the
optimal inductive solution of a task can lead to a better quality hypothesis
than approximating the optimal solution.
Note that for four out of the five datasets, Inspire performs better with 100
examples than with 500 examples. A possible explanation for this is that with
more examples, Inspire does not get as close to the optimal solution as it does
with fewer examples, thus leading to a lower $F_1$ score on the test data. With
500 examples, ILASP3 does take longer to terminate than it does for 100
examples, but in four out of the five cases, ILASP's average $F_1$ score is
higher, confirming the expected result that more data should tend to lead to a better
hypothesis.

\subsection{Car Preference Learning}
\label{sec:car}

We tested ILASP3's ability to learn real user preferences with the \emph{car
preference dataset} from~\cite{abbasnejad2013learning}. This dataset consists
of responses from 60 different users, who were each asked to give their
preferences about ten cars. They were asked to order each (distinct) pair of
cars, leading to 45 orderings. The cars had four attributes, shown in
\tblref{car}~(a). Through this experiment, we aim to show that ILASP3 is
capable of learning real user preferences, encoded as weak constraints.  There
is not much work on applying ILP systems to preference learning, but one such
work~\citep{qomariyah2017learning} applied the Aleph~\citep{srinivasan2001aleph}
system to the car preference dataset. Aleph is not guaranteed to find an
optimal solution,\footnote{Aleph processes the examples sequentially, and
searches for the best clause to add in each iteration (in terms of coverage).
Although each iteration adds the best clause, this may still lead to a
sub-optimal hypothesis overall.} and is only capable of learning rules (and not
of learning weak constraints).  \cite{qomariyah2017learning} used Aleph to
learn rules defining the predicate $\asp{bt/2}$, where $\asp{bt(c_1, c_2)}$
represents that $\asp{c_1}$ is preferred to $\asp{c_2}$. For comparison, we
present the results of~\cite{qomariyah2017learning} on this dataset.

\begin{table}[t]
  \vskip -0.15in

\caption{(a) The attributes of the car preference dataset, along with the
  possible range of values for each attribute. The integer next to each value
is how that value is represented in the data. (b) The accuracy results of
ILASP3 compared with the three methods in~\cite{qomariyah2017learning} on the
  car preference dataset. \label{tbl:car}}

  \small

  \centering

  \begin{multicols}{2}

  \begin{tabular}{|c|c|}

  \hline

  Attribute & Values\\\hline

    Body type & $\asp{sedan} (1)$, $\asp{suv} (2)$\\
    Transmission & $\asp{manual} (1)$, $\asp{automatic} (2)$\\
    Engine Capacity & $\asp{2.5L}$, $\asp{3.5L}$, $\asp{4.5L}$, $\asp{5.5L}$, $\asp{6.2L}$\\
    Fuel Consumed & $\asp{hybrid} (1)$, $\asp{non\_hybrid} (2)$\\\hline

  \end{tabular}

  \begin{tabular}{|c|c|}
    \hline Method & Accuracy\\\hline
     SVM & 0.832\\
     DT & 0.747\\
     Aleph & 0.729\\
     ILASP3 A & 0.880\\
     ILASP3 B & 0.863\\\hline
  \end{tabular}
  \end{multicols}

  \begin{multicols}{2}

    (a)

    (b)

  \end{multicols}

  \vskip -0.06in

\end{table}

Our initial experiment was based on an experiment
in~\cite{qomariyah2017learning}, where the Aleph~\citep{srinivasan2001aleph}
system was used to learn the preferences of each user in the dataset and
compared with support vector machines (SVM) and decision trees (DT). Ten-fold
cross validation was performed for each of the 60 users on the 45 orderings. In
each fold, 10\% of the orderings were omitted from the training data and used
to test the learned hypothesis. The flaw in this approach is that
the omitted examples will often be implied by the rest of the examples (i.e. if $a
\prec b$ and $b \prec c$ are given as examples it does not make sense to
omit $a \prec c$). For this reason, we also experimented with leaving out all
the examples for a single car in each fold (i.e. every pair that contains that
car), and using these examples to test (again leading to ten folds).
This new task corresponds to learning preferences from a complete ordering of
nine cars, and testing the preferences on an unseen car.

Table~\ref{tbl:car}~(b) shows the accuracy of the approach
in~\cite{qomariyah2017learning} and ILASP3 accuracy on the two versions of the
experiment. The easier task (with 10\% of the orderings omitted) is denoted as
experiment A in the table, and the harder task is denoted as experiment B. In
fact, even on the harder version of the task, ILASP3 performs better than the
approaches in \cite{qomariyah2017learning} perform on the easier version of the
task.
In one fold for the first user (in experiment A),
ILASP3 learns the following weak constraints:
$\asp{:\sim fuel(2).[1@4]}$;\;
$\asp{:\sim body(1), transmission(2).[-1@3]}$;\;
$\asp{:\sim engine\_cap(V0).[V0@2, V0]}$;\;
$\asp{:\sim body(1).[-1@1]}$.
This hypothesis corresponds to the following set of prioritised preferences
(ordered from most to least important): the user (1) prefers hybrid cars to
non-hybrid cars; (2) likes automatic sedans; (3) would like to minimise the
engine capacity of the car; and (4) prefers sedans to SUVs.

The noise in this experiment comes from the fact that some of the answers given
by participants in the survey may contradict other answers. Some participants
gave inconsistent orderings (breaking transitivity) meaning that there is no
set of weak constraints that covers every ordering example.

The results of these experiments have shown that ILASP3 is able to learn
hypotheses that accurately represent real user preferences, even in the
presence of noise. On average, ILASP3 learns a hypothesis with a higher
accuracy than the hypothesis learned by~\cite{qomariyah2017learning}. This
could be for two reasons: (1) the fact that Aleph might return a sub-optimal
inductive solution; or (2), the representation of hypotheses as weak
constraints allows for preferences to be expressed that cannot be expressed
using the definite search space in~\cite{qomariyah2017learning}.

\subsection{SUSHI Preference Learning}
\label{sec:sushi}

\begin{table}[t]
\vskip -0.15in

  \caption{\label{tbl:sushi} (a) the attributes of the SUSHI preference
  dataset, along with the range of values for each attribute, and (b) the
  average accuracy of ILASP3 compared with the methods used
  in~\cite{qomariyah2017learning}.}

  \centering

  \small

  \begin{multicols}{2}

  \begin{tabular}{|c|c|}

  \hline

  Attribute & Values\\\hline

  Style & $\asp{maki} (\asp{0})$, $\asp{non\_maki} (\asp{1})$\\
  Major group & $\asp{seafood} (\asp{0})$, $\asp{non\_seafood} (\asp{1})$\\
  Minor group & $\asp{0,\ldots, 11}$\\
  Oiliness & $[0, 4]$\\
  Frequency Eaten & $[0, 3]$\\
  Normalised Price & $[0, 5]$\\
  Frequency Sold & $[0, 1]$\\\hline

  \end{tabular}

  \begin{tabular}{|c|c|}
    \hline Method & Accuracy\\\hline
    SVM & 0.76\\
    DT &  0.81 \\
    Aleph & 0.78\\
    ILASP3 & 0.84\\\hline
  \end{tabular}

  \end{multicols}

  \begin{multicols}{2}

    (a)

    (b)

  \end{multicols}

  \vskip -0.06in

\end{table}

Another dataset for preference learning is the SUSHI
dataset~\citep{kamishima2010survey}. The dataset is comprised of peoples'
preference orderings over different types of sushi.  The purpose of these
experiments is to show that ILASP3 is capable of learning weak constraints that
accurately capture real user preferences. \cite{qomariyah2017learning} also
tested their approach on these datasets, and we compare ILASP3's accuracy with
their results in order to test whether the optimal solution found by ILASP3 is
more accurate than their solutions.

Each type of sushi has several attributes, described in
\tblref{sushi}~(a).  There is a mix of categorical and continuous attributes.
In the language bias for these experiments, the categorical attributes are used
as constants, whereas the continuous attributes are variables that can be used
as the weight of the weak constraint. This allows weak constraints to express
that the continuous attributes should be minimised or maximised.
The dataset was constructed from a survey in which people were asked to order
ten different types of sushi. This ordering leads to 45 ordering examples per
person. This experiment is based on a similar experiment
in~\cite{qomariyah2017learning}. For each of the first 60 people in the
dataset ten-fold cross validation was performed, omitting 10\% of the orderings
in each fold. This experiment suffers from the same flaw as Experiment A on the
car dataset in that some of the omitted examples may be implied by the training
examples, but we give the results for a comparison
to~\cite{qomariyah2017learning}. As shown in \tblref{sushi}~(b), ILASP3
achieved an average accuracy of 0.84, comparing favourably to each result
from~\cite{qomariyah2017learning}.

Although in this experiment each participant gave a consistent total ordering
of the ten types of sushi, it might be the case that there is no hypothesis in
the hypothesis space that covers all of the examples. This could be the case
when we are not modelling a feature of the sushi that the participant considers
to be important. For this reason, we treated this as a noisy learning setting,
and used ILASP3 to maximise the coverage of the examples.

This experiment has shown that ILASP3 is capable of learning weak constraints
that accurately capture users' preferences, and that ILASP3's approach of
finding an optimal hypothesis comprising of weak constraints is (on average)
more accurate than the approach of~\cite{qomariyah2017learning}, which finds a
(potentially sub-optimal) set of definite clauses.

\subsection{Comparison to $\delta$ILP}
\label{sec:dilp}

Although the work in this paper concerns learning ASP programs from noisy
examples, work has been done in the area of extending definite clause learning
to handle noisy examples (for example,
\cite{sandewall1993handling,srinivasan2001aleph,hyper_n}).
In~\cite{evans2017learning}, it was claimed that ILP approaches are unable ``to
handle noisy, erroneous, or ambiguous data'' and that  ``If the positive or
negative examples contain any mislabelled data, [ILP approaches] will not be
able to learn the intended rule''.  The experiments in this section aim to
refute this claim.

To learn from noisy data, \cite{evans2017learning} introduced the $\delta$ILP
algorithm, based on artificial neural networks. They demonstrated that
$\delta$ILP is able to achieve a high accuracy even with a large proportion of
noise in the examples. \cite{evans2017learning} evaluated $\delta$ILP on six
synthetic datasets, with noise ranging from 0\% to 90\%. In these experiments,
we investigated the accuracy of ILASP3 on five of these six
datasets.\footnote{The authors of~\citep{evans2017learning} provided us with
the training and test data for these five problems.} In the original
experiments, examples were atoms, and noise corresponded to swapping positive
and negative examples. In each of the $ILP_{LOAS}^{noise}$ tasks, we ensured
that the hypothesis space was such that for each $H\subseteq S_M$, $B\cup H\cup
e_{ctx}$ was stratified for each example $e$. This allowed atomic examples to
be represented as (positive) partial interpretations -- a positive example
$\asp{e}$ was represented as a partial interpretation $\langle
\answerset{e},\emptyset\rangle$, and a negative example $\asp{e}$ was
represented as a partial interpretation $\langle \emptyset,
\answerset{e}\rangle$.  Due to the differences in language biases used by ILASP
and $\delta$ILP, the hypothesis spaces of the two systems are not equivalent.

Due to the imbalance of positive and negative examples in many of the tasks, we
weight the positive examples at $w \times |E^-|/(|E^+|+|E^-|)$ and the negative
examples at $w \times |E^+|/(|E^+|+|E^-|)$, where in this experiment $w$ is
$100$. 
The weight for
each example class (positive or negative) is equal to $w$ multiplied by the
proportion of the whole set of examples which are in the other class. This
``corrects'' any imbalance between positive and negative examples (i.e.\ the
penalty for not covering a proportion of the positive examples is the same as
the penalty for not covering the same proportion of negative examples).  The
constant $w$ can be thought of as the difference in importance between the
hypothesis length and the number of examples covered. In these experiments we
chose 100, as it is high enough to ensure that coverage is considered far more
important than hypothesis length.

%

\begin{figure*}[t]

  \hfill
    \includegraphics[width=0.28\textwidth, trim={23mm 0mm 18mm 2mm},clip]{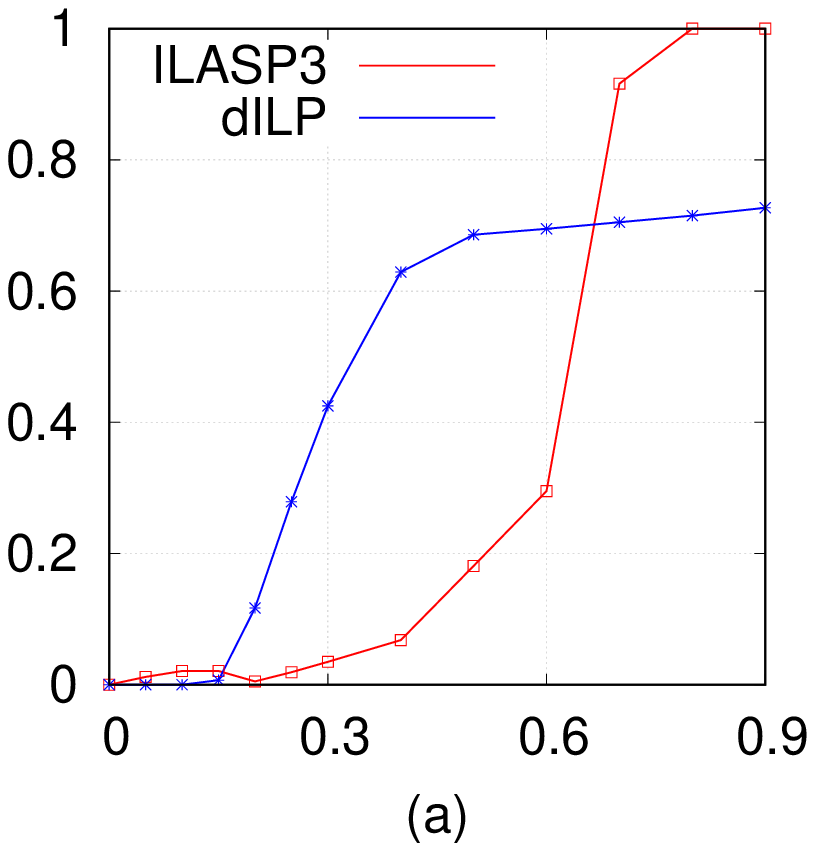}
  \hfill
    \includegraphics[width=0.28\textwidth, trim={23mm 0mm 18mm 2mm},clip]{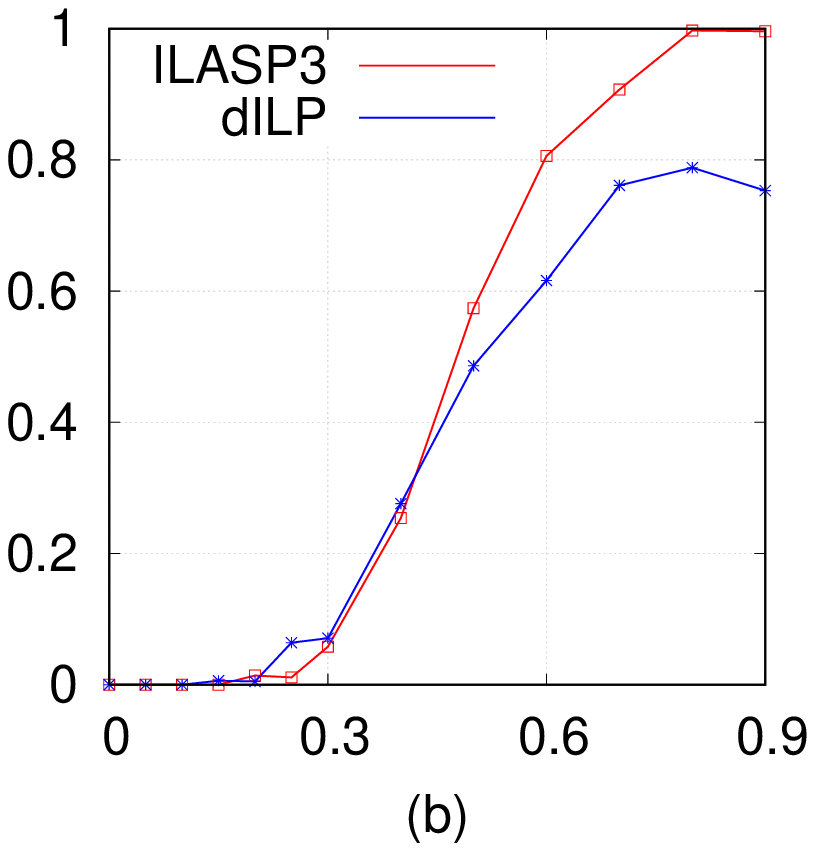}
  \hfill
    \includegraphics[width=0.28\textwidth, trim={23mm 0mm 18mm 2mm},clip]{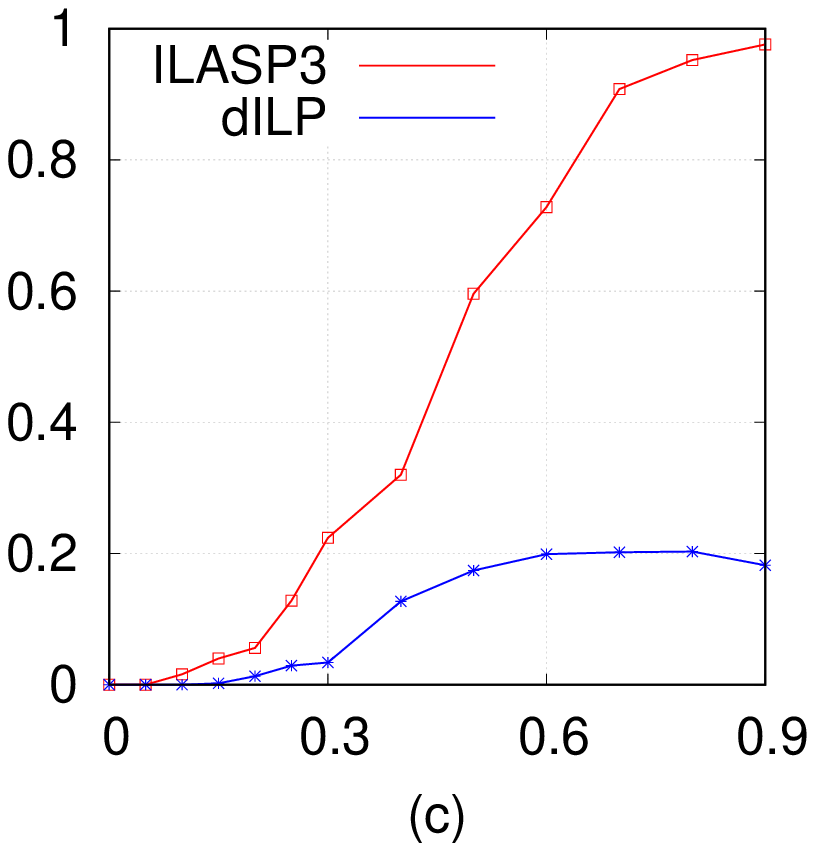}
  \hfill
  \!

  \vspace{2mm}

  \hfill
    \includegraphics[width=0.28\textwidth, trim={23mm 0mm 18mm 2mm},clip]{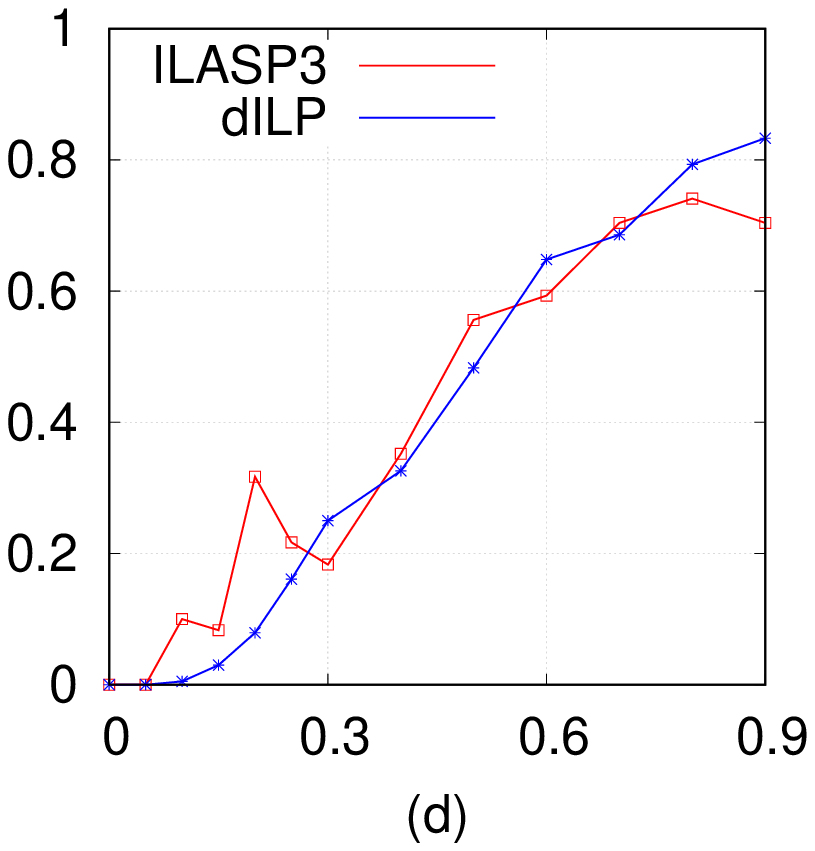}
  \hfill
    \includegraphics[width=0.28\textwidth, trim={23mm 0mm 18mm 2mm},clip]{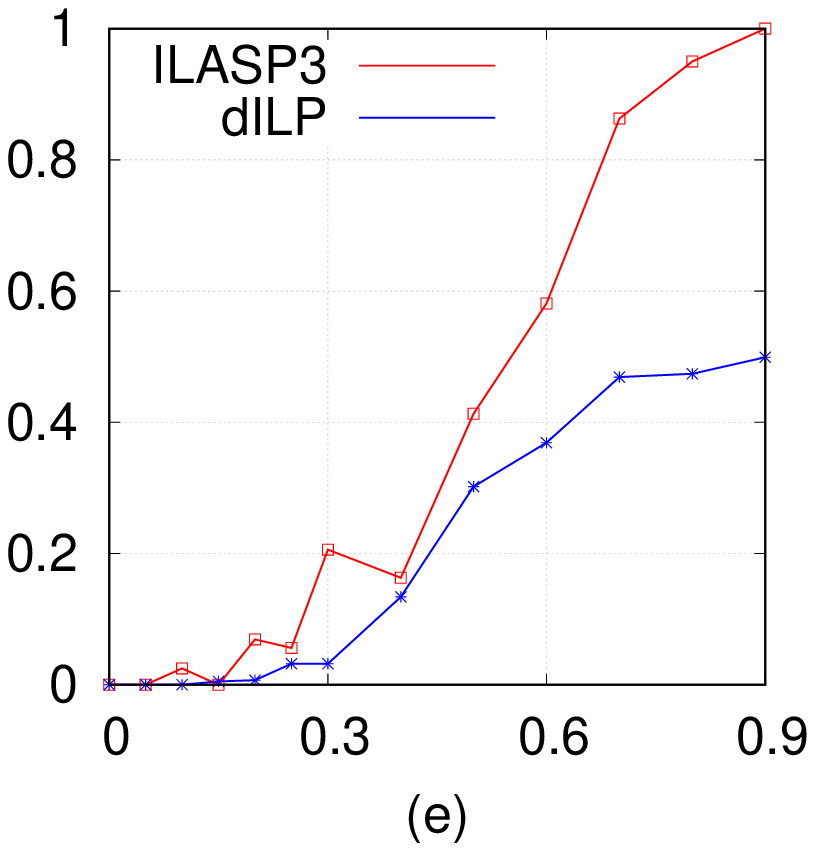}
  \hfill
  \!

    \caption{\label{fig:richard_error} A comparison of $\delta$ILP
    and ILASP3 on five datasets from \cite{evans2017learning}. Specifically the
    graphs correspond to the (a) predecessor, (b) less than, (c) member, (d)
    connected and (e) undirected edge experiments in~\cite{evans2017learning}.
    In each graph, the X and Y axes represent the noise level and mean squared
    error, respectively.}

\end{figure*}

Figure~\ref{fig:richard_error} shows the mean squared error of the two systems, where the results for $\delta$ILP are taken from \cite{evans2017learning}.  
In most tasks ILASP3 achieves similar results to $\delta$ILP when the noise is in the range of 0\% to 40\%. However,
at the other end of the scale (with more than 50\% noise), there are some tasks
where ILASP3 finds hypotheses with close to 100\% error, where $\delta$ILP's
error is much lower (less than 20\% in the ``member'' problem). We argue
that when the noisy examples outnumber the correctly labelled examples, the
learner should start learning the negation of the target hypothesis; for
instance, in the case of ``less than'', ILASP3 correctly learned the ``greater
than or equal to'' relation.
%
%
The ideal outcome of these kinds of experiments, where the proportion of noise
is varied, is that the learner achieves close to 0\% error until around 50\%
noise and close to 100\% error thereafter. This is roughly what seems to happen
for ILASP3 in the ``predecessor'', ``less than'', ``member'' and ``undirected
edge'' experiments. In ``predecessor'', the graph is less symmetric, with the
``crossover'' from low to high error occurring later. This is likely because
the hypothesis for ``not predecessor'' is longer than the hypothesis for
``predecessor''.  The failure of $\delta$ILP to get close to 100\% error in
many of the tasks (for example in ``member'', $\delta$ILP has an error
of less than 20\% with the noise level at 90\%) may indicate that the negation
of the target concept is not representable given the language bias used by
$\delta$ILP in these experiments, instead of $\delta$ILP being particularly
robust to noise. In some cases (such as ``member''), this is likely because the
negation of the concept requires negation as failure (which is not supported by
$\delta$ILP), but in others such as ``less than'', the negation of the concept
is expressible without negation as failure.

These results show that, on the ILP problems investigated
by~\cite{evans2017learning}, ILASP3 is certainly robust to noise, thus refuting
their claim that ILP systems cannot handle noise.

\section{Related Work}

Several other ILP systems use ASP solvers in the search for hypotheses. For
example, \cite{action_desc_asp} presented an early system for learning action
descriptions, where the search for inductive solutions is encoded in ASP. Many
of these systems, such as
\cite{action_desc_asp,raspal,bragaglia2015nonmonotonic,inspire} operate under a
\emph{brave} semantics -- the learned program should have at least one answer
set that satisfies some given properties (such as covering examples). But our
results on the generality of learning frameworks in~\cite{AIJ17} prove that
there are ASP programs that can be learned by our framework and that cannot be
learned by any of these systems. For example, brave induction systems cannot
learn hard or weak constraints, no matter what examples are given.

In a different line of research \citep{rrl_asp,rrl_asp2}, ASP solvers have also
been used together with relational reinforcement learning (RRL).
\cite{rrl_asp} present an architecture that combines RRL with ASP-based
inference. RRL and decision tree induction were
used to identify a set of candidate axioms. The candidates deemed to have the
highest likelihood are then represented in an ASP program, which is used for
planning.

Early approaches to relational learning (e.g.\
\cite{langley1987general,theory_rev_noise,cohen95}) were able to learn definite
rules from noisy data.
\cite{theory_rev_noise} presented
an ILP system based on theory revision, where hypotheses are only modified if
the modification leads to the additional coverage of more than one example. In
practice, however, it is possible that given a large enough set of examples,
two noisy examples may be covered by exactly the same class of hypotheses.
Under the $ILP_{LOAS}^{noise}$ approach, the penalty for not covering a set of
examples which forms a small proportion of examples is low, even if there are
multiple examples in this set. \cite{cohen95} introduces algorithms which learn
from noisy examples, learning one clause at a time. ILP systems which
iteratively learn single clauses, removing covered positive examples after each
iteration, are common when the target hypotheses are definite logic programs
(with no negation), as the programs being learned are \emph{monotonic}.
Learning \emph{non-monotonic} ASP programs with negation (allowing for the
learning of exceptions) requires a different approach
(\cite{ray2009nonmonotonic}). This is because, due to the non-monotonicity of
the learned programs, examples which are covered in one iteration may become
uncovered when further rules are learned.

In order to search for good hypotheses, ILP systems often use a cost function,
defined in terms of the coverage of the examples and the length of the
hypothesis (e.g.\
\cite{srinivasan2001aleph,muggleton1995inverse,bragaglia2015nonmonotonic}).
When examples are noisy, this cost function is sometimes combined with a notion
of maximum threshold, and the search is not for a hypothesis that minimises the
cost function, but for a hypothesis that does not fail to cover more than a
defined maximum threshold number of examples (e.g.
\cite{srinivasan2001aleph,hyper_n,raspal}).  In this way, once an acceptable
hypothesis (i.e.\ a hypothesis that covers a sufficient number of examples) is
computed the system does not search for a better one. As such, the
computational task is simpler, and therefore the time needed to compute a
hypothesis is shorter, but there may be other hypotheses which have a lower cost.
Furthermore, to guess the ``correct'' maximum threshold requires some idea of
how much noise there is in the given set of examples. For instance, one of the
inputs to the HYPER/N~\citep{hyper_n} system is the proportion of noise in the
examples. When the proportion of noise is unknown, too small a threshold could
result in the learning task being unsatisfiable, or in learning a hypothesis
that overfits the data. On the other hand, too high a threshold could result in
poor accuracy, as the hypothesis may not cover many of the examples.  Our
$ILP_{LOAS}^{noise}$ framework addresses the problem of computing optimal
solutions (with respect to the cost function) and in doing so does not require
knowledge a priori of the level of noise in the data.
Note that optimal hypotheses are not \emph{guaranteed} to outperform other
hypotheses on unseen data, but based on the evidence (i.e.\ the training
examples) they minimise the cost function, and so if the cost function is
reasonable, they should be more likely to be correct.  This can be seen in the
sentence chunking experiments, where we used ILASP with the same cost function
as Inspire (which does not guarantee minimising the cost function).  In future
work, we intend to explore alternative cost functions, and formalise what makes
a cost function ``reasonable'' in a given learning setting.

\section{Conclusion}

Learning interpretable knowledge is a key requirement for cognitive systems
that are required to communicate with each other, or with humans. Our research
addresses the problem of learning ASP programs, which are capable of
representing complex knowledge, such as defaults, exceptions and preferences.
In practice, cognitive systems are required to learn knowledge from noisy data
sources, where there is no guarantee that all examples are perfectly labelled.

This paper has presented the $ILP_{LOAS}^{noise}$ framework for learning ASP
from noisy examples and evaluated the ILASP3 system, designed to solve the
learning tasks of this framework. We used several synthetic datasets to show
that ILASP3 can learn even in the presence of high proportions noisy examples.
We have also tested ILASP3's performance on several datasets used by other ILP
systems The results of these experiments show that in most cases ILASP3 is able
to learn with a higher accuracy than the other systems, which, unlike ILASP3,
are not guaranteed to find optimal solutions of the tasks.

Although ILASP3 is a significant improvement on previous ILASP systems with
respect to the running time on noisy tasks, some scalability issues remain,
especially with the size of the hypothesis space. Every ILASP system begins by
computing the hypothesis space in full, which limits the feasible size of the
hypothesis space. In future work, we plan to design ILASP systems which do not
begin by computing the hypothesis space in full.

\begin{acknowledgements}
\noindent
We would like to thank the reviewers for their useful comments and suggestions.
\end{acknowledgements}

\vspace{-0.25in}

{\parindent -10pt\leftskip 10pt\noindent
\bibliographystyle{cogsysapa}
\bibliography{ms}

}

\clearpage

\appendix

\section{Details of the Representations used in the Hamilton Graph and Journey Preference Experiments}
\label{sec:representations}

\subsection{Hamiltonian Graphs}

The Hamilton graph learning tasks in \secref{hamilton} were aimed at learning
how to decide whether a graph is Hamiltonian of not.  The four node Hamiltonian
graph $G$ in figure~\ref{fig:ham} can be represented by the set of facts
$F_{G}$.

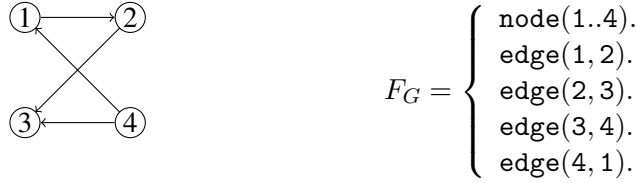
\begin{figure}[h]
  \begin{multicols}{2}
  \begin{center}

    \begin{tikzpicture}[scale=1.4,regular/.style={draw,ellipse, inner sep=0,minimum size=4mm}]
      \draw (1,2) node [regular] {1};
      \draw (2,2) node [regular] {2};
      \draw (1,1) node [regular] {3};
      \draw (2,1) node [regular] {4};
      \draw [->] (1.15,2) to (1.85,2);
      \draw [->] (1.90,1.90) to (1.10,1.10);
      \draw [->] (1.85,1) to (1.15,1);
      \draw [->] (1.90,1.10) to (1.10,1.90);
    \end{tikzpicture}
  \end{center}

\begin{math}
  F_{G} = \left\{
  \begin{array}{l}
  \asp{node(1\ruleend\ruleend4)\ruleend}\\
  \asp{edge(1,2)\ruleend}\\
  \asp{edge(2,3)\ruleend}\\
  \asp{edge(3,4)\ruleend}\\
  \asp{edge(4,1)\ruleend}
\end{array}\right.
\end{math}

  \end{multicols}
  \caption{\label{fig:ham} An example of a Hamiltonian Graph $G$ and its corresponding representation in ASP, $F_{G}$.}
\end{figure}

To decide whether a graph is Hamiltonian or not, we can use the program $H$
below:

{\small
\begin{verbatim}
reach(V0) :- in(1,V0).
reach(V1) :- in(V0,V1), reach(V0).
0 {in(V0,V1) } 1 :- edge(V0, V1).
:- node(V0), not reach(V0).
:- in(V0,V1), in(V0,V2), V1 != V2.
\end{verbatim}}

If for a graph $G$ and its corresponding  set of facts $F_{G}$, $G$ is
Hamiltonian if and only if $F_{G}\cup H$ is satisfiable.
In the $ILP_{LOAS}^{noise}$ tasks, we made use of (weighted) CDPI examples to
represent graphs. We did not use any background knowledge (i.e.\ the background
knowledge in each task was empty), and instead encoded the graphs in the
contexts of examples such as in the positive CDPI example:
$\langle\langle\emptyset,\emptyset\rangle,\lbrace
\asp{node(1\ruleend\ruleend4)\ruleend}\;\; \asp{edge(1,2)\ruleend}\;\;
\asp{edge(2,3)\ruleend}\;\; \asp{edge(3,4)\ruleend}\;\; \asp{edge(4,1)\ruleend}
\rbrace\rangle$, which represents the graph $G$.

\subsection{Journey Preferences}

We now describe the structure of journeys in the experiments in \secref{jp}. A
journey consists of a set of \emph{legs}. The attributes of journey legs in
these experiments were: $\asp{mode}$, which took one of the values $\asp{bus}$,
$\asp{car}$, $\asp{walk}$ or $\asp{bicycle}$; $\asp{distance}$, which took an
integer value between $\asp{1}$ and $\asp{20,000}$; and $\asp{crime\_rating}$.
As the crime ratings were not readily available from the simulator, we used a
randomly generated value between $\asp{1}$ and $\asp{5}$ for each journey leg.

In the experiments, we assumed that a user's preferences could be represented
by a set of weak constraints based on the attributes of a leg. $S_J$ denotes
the set of possible weak constraints that we used in the experiments, each of
which includes at most three literals (characterised by a mode bias which can be
found at \url{https://www.doc.ic.ac.uk/~ml1909/ILASP/}).
Most of these literals capture the leg's attributes, e.g., $\asp{mode(L, bus)}$
or $\asp{crime\_rating(L, R)}$ (if the attribute's values range over integers
this is represented by a variable, otherwise each possible value is used as a
constant). For the crime rating ($\asp{crime\_rating(L,R)}$), we also allow
comparisons of the form $R > \asp{c}$ where $\asp{c}$ is an integer from 1 to
4. The weight of each weak constraint is a variable representing the distance
of the leg in the body of the weak constraint, or 1 and the priority is 1, 2 or
3. One possible set of preferences is represented by the weak constraints
$W^*$.

$$W^* = \left\{\begin{array}{l}
  \asp{:\sim leg\_mode(L, walk), leg\_crime\_rating(L, C), C > 3\ruleend[1@3, L, C]}\\
  \asp{:\sim leg\_mode(L, car)\ruleend[1@2, L]}\\
  \asp{:\sim leg\_mode(L, walk), leg\_distance(L, D)\ruleend[D@1, L, D]}\\
\end{array}\right\}$$

These preferences represent that the user's top priority is to avoid walking
through areas with a high crime rating. Second, the user would like to avoid
driving, and finally, the user would like to minimise the total walking
distance of the journey.

We now describe how to represent the journey preferences scenario in
$ILP_{LOAS}^{noise}$. We assume that each journey is encoded as a set of
attributes of the legs of the journey; for example the journey $\answerset{
  distance(leg(1), 2000); distance(leg(2), 100); mode(leg(1), bus);
  mode(leg(2), walk)}$\break has two legs; in the first leg, the person must take a
  bus for 2,000m and in the second, he/she must walk 100m.
  Each of our learning tasks had an empty background knowledge. Each positive
  example in our tasks was a weighted CDPI $\langle e_{id}, \infty,
  \langle\langle\emptyset,\emptyset\rangle, J\rangle\rangle$, where $J$ is the
  set of facts representing a journey. The brave ordering examples were defined
  over pairs of the positive examples with appropriate ordering operators, and
  each with a penalty of 1.  Note that the positive examples are automatically
  satisfied as the (empty) background knowledge (combined with the context)
  already covers them. Also, as the background knowledge together with each
  context has exactly one answer set, the notions of brave and cautious
  orderings coincide; hence, we do not need cautious ordering examples for this
  task. Furthermore, since only weak constraints are being learned, the task
  also has no negative examples (a negative example would correspond to an
  invalid journey).

\section{ILASP Flags Used in the Experiments}\label{sec:flags}

ILASP3 has various optional features, used to improve the speed of the
algorithm on different kinds of learning task. \tblref{flags} shows the option
flags that were used in the calls to ILASP in each experiment. In addition to
these ``core'' ILASP options, in all but the Hamilton and CAVIAR experiments, a
flag was passed to run Clingo 5 with the option
\verb|--opt-strat=usc,stratify|.

\begin{table}[t]

  \vskip -0.1in

  \caption{\label{tbl:flags} The flags that were passed to ILASP when running
    the experiments in this paper.}

  \vspace{0.254cm}

  \centering

  \small

  \begin{tabular}{|c|c|}

  \hline

  Experiment & Flags\\\hline\hline

  Hamilton Graphs & \verb|-ng|\\\hline
  Journey Preferences & \verb|-ng -swc|\\\hline
  CAVIAR & \verb|-np -ng|\\\hline 
  Sentence Chunking & \verb|-np -ng --max-translate|\\\hline 
  Cars & \verb|-np -ni -ng -swc|\\\hline
  Sushi & \verb|-np -ni -ng -swc|\\\hline
    $\delta$ILP datasets & (no extra options)\\\hline

  \end{tabular}

  \vskip -0.06in

\end{table}

\section{Proofs}

In this section, we prove \thmref{complexity}, showing that
$ILP_{LOAS}^{noise}$ shares the same computational complexity as
$ILP_{LOAS}^{context}$ on the two decision problems of verification and
satisfiability. The proof relies on two propositions.

\begin{proposition}\label{prop:ctn}
  Deciding verification and satisfiability for a propositional
  $ILP_{LOAS}^{context}$ task both reduce polynomially to the same problem for
  a propositional $ILP_{LOAS}^{noise}$ task.
\end{proposition}

\begin{proof}
Let $T$ be the $ILP_{LOAS}^{context}$ task $\langle B, S_M, \langle E^+,
E^-, O^b, O^c\rangle\rangle$. Consider the $ILP_{LOAS}^{noise}$ task $T'
= \langle B, S_M, \langle E^+_2, E^-_2, O^b_2, O^c_2\rangle\rangle$,
where the examples are defined as follows:

\begin{itemize}
  \item $E^+_2 = \left\{ \langle e_{id}, \infty, e\rangle \middle| e \in E^+\right\}$
  \item $E^-_2 = \left\{ \langle e_{id}, \infty, e\rangle \middle| e \in E^-\right\}$
  \item $O^b_2 = \left\{ \langle o_{id}, \infty, o\rangle \middle| o \in O^b\right\}$
  \item $O^c_2 = \left\{ \langle o_{id}, \infty, o\rangle \middle| o \in O^c\right\}$
\end{itemize}

First note that $H \in ILP_{LOAS}^{context}(T) \Leftrightarrow \mathcal{S}(H,
T')$ is finite. Hence $H \in ILP_{LOAS}^{context}(T) \Leftrightarrow
H \in ILP_{LOAS}^{noise}(T')$. So verification for $ILP_{LOAS}^{context}$
reduces to verification for $ILP_{LOAS}^{noise}$. As this also means that
$ILP_{LOAS}^{context}(T) = \emptyset \Leftrightarrow ILP_{LOAS}^{noise}(T') =
\emptyset$, this also shows that satisfiability for $ILP_{LOAS}^{context}$
reduces to satisfiability for $ILP_{LOAS}^{noise}$.

\end{proof}

\begin{proposition}\label{prop:ntc}
  Deciding verification and satisfiability for a propositional
  $ILP_{LOAS}^{noise}$ task both reduce polynomially to the same problem for a
  propositional $ILP_{LOAS}^{context}$ task.
\end{proposition}

\begin{proof}

Let $T$ be the $ILP_{LOAS}^{noise}$ task $\langle B, S_M, \langle
E^+, E^-, O^b, O^c\rangle\rangle$. Consider the
$ILP_{LOAS}^{context}$ task $T' = \langle B, S_M, \langle E^+_2,
E^-_2, O^b_2, O^c_2\rangle\rangle$, where the examples are defined
as follows:

\begin{itemize}
  \item $E^+_2 = \left\{ \langle e_{pi}, e_{ctx}\rangle \middle| \langle
    e_{id}, e_{pen}, \langle e_{pi}, e_{ctx}\rangle\rangle \in E^+, e_{pen} =
    \infty\right\}$
  \item $E^-_2 = \left\{ \langle e_{pi}, e_{ctx}\rangle \middle| \langle
    e_{id}, e_{pen}, \langle e_{pi}, e_{ctx}\rangle\rangle \in E^-, e_{pen} =
    \infty\right\}$
  \item $O^b_2 = \left\{ o_{ord} \middle| \langle o_{id}, o_{pen},
    o_{ord}\rangle \in O^b, o_{pen} = \infty\right\}$
  \item $O^c_2 = \left\{ o_{ord} \middle| \langle o_{id}, o_{pen},
    o_{ord}\rangle \in O^c, o_{pen} = \infty\right\}$
\end{itemize}

$\forall H \subseteq S_M$, $H \in ILP_{LOAS}^{context}(T')$ if and only if $H$
covers all examples in $T$ that have a finite penalty.  Hence,
$ILP_{LOAS}^{context}(T') = ILP_{LOAS}^{noise}(T)$. This means that both
verification and satisfiability for $ILP_{LOAS}^{noise}$ reduce to verification
and satisfiability for $ILP_{LOAS}^{context}$ (as $H \in
ILP_{LOAS}^{noise}(T)\Leftrightarrow H \in ILP_{LOAS}^{context}(T')$ and
$ILP_{LOAS}^{noise}(T)=\emptyset\Leftrightarrow
ILP_{LOAS}^{context}(T')=\emptyset$).

\end{proof}

\pagebreak{}

\noindent We can now prove \thmref{complexity}.

\vspace{2mm}

\noindent \textbf{\thmref{complexity}}
\begingroup $ $

\begin{enumerate}
  \item Deciding verification for an arbitrary propositional $ILP_{LOAS}^{noise}$ task is $DP$-complete
  \item Deciding satisfiability for an arbitrary propositional $ILP_{LOAS}^{noise}$ task is $\Sigma^P_2$-complete
\end{enumerate}
\endgroup

\begin{proof}{$ $}

  \begin{enumerate}
    \item
      As Propositions~\ref{prop:ctn}~and~\ref{prop:ntc} show polynomial
      reductions in both directions from this problem to deciding verification
      for an arbitrary propositional $ILP_{LOAS}^{context}$ task, it remains to
      show that the corresponding decision problem for $ILP_{LOAS}^{context}$
      is $DP$-complete. This was shown in~\citep{AIJ17}.
    \item
      Similarly, as Propositions~\ref{prop:ctn}~and~\ref{prop:ntc} show
      polynomial reductions in both directions from this problem to deciding
      satisfiability for an arbitrary propositional $ILP_{LOAS}^{context}$
      task, it remains to show that the corresponding decision problem for
      $ILP_{LOAS}^{context}$ is $DP$-complete. This was shown in~\citep{AIJ17}.
  \end{enumerate}

\end{proof}

\end{document}